\documentclass{article}
\usepackage{latexsym}
\usepackage{amsmath}
\usepackage{amsthm}
\usepackage{amsfonts}
\usepackage{url}
\usepackage{amssymb}
\usepackage{mytheo}
\usepackage{geometry}
\usepackage{graphicx}
\usepackage{color}

\usepackage{ifpdf}
\ifpdf    %
\usepackage{hyperref}
\else    %
\usepackage[hypertex]{hyperref}
\fi

\newcommand{\myparagraph}[1]{\medskip\paragraph*{#1}}

\def\balig{\beq}
\def\ealig{\eeq}
\def\balign{\beqn}
\def\ealign{\eeqn}

\newcommand{\powddim}{^{\ddim(\X)}}

\renewcommand{\P}{\mathbb{P}}

\newcommand{\mexp}{\mathbb{E}}

\newcommand{\E}{\mexp}
\def\clap#1{\hbox to 0pt{\hss#1\hss}}

\newcommand{\trn}{^{\!\mathsf{T}}}

\newcommand{\Lip}[1]{\nrm{#1}_{\textrm{{\tiny \textup{Lip}}}}}

\newcommand{\sgn}{\operatorname{sgn}}
\newcommand{\polylog}{\operatorname{polylog}}

\newcommand{\evalat}[2]{{#1}_{|#2}}
\makeatletter
\newcommand{\ben}{\begin{enumerate}}
\newcommand{\een}{\end{enumerate}}
\newcommand{\bit}{\begin{itemize}}
\newcommand{\eit}{\end{itemize}}

\newcommand{\nrm}[1]{\left\Vert #1 \right\Vert}

\newcommand{\tsnrm}[1]{\Vert #1 \Vert}
\newcommand{\iprod}[1]{\left\langle #1 \right\rangle}

\newcommand{\calF}{\mathcal{F}}

\newcommand{\R}{\mathbb{R}}
\newcommand{\N}{\mathbb{N}}

\newcommand{\beq}{\begin{eqnarray*}}
\newcommand{\eeq}{\end{eqnarray*}}
\newcommand{\beqn}{\begin{eqnarray}}
\newcommand{\eeqn}{\end{eqnarray}}
\newcommand{\paren}[1]{\left( #1 \right)}

\newcommand{\tlprn}[1]{\left\{ #1 \right\}}
\newcommand{\set}[1]{\tlprn{#1}}
\newcommand{\abs}[1]{\left| #1 \right|}
\newcommand{\tsabs}[1]{| #1 |}
\newcommand{\ceil}[1]{\ensuremath{\left\lceil#1\right\rceil}}

\newcommand{\gn}{\, | \,}

\newcommand{\hide}[1]{}

\def\eps{\varepsilon}

\newcommand{\ftil}{\tilde f}
\newcommand{\trun}[2]{{\operatorname{T}}_{[{#1},{#2}]}}

\newcommand{\citep}[1]{\cite{#1}}
\newcommand{\citet}[1]{\cite{#1}}

\newcommand{\X}{\mathcal{X}}

\newcommand{\F}{\mathcal{F}}

\newcommand{\ddim}{\mathrm{ddim}}
\newcommand{\diam}{\mathrm{diam}}
\newcommand{\ERP}{\mathrm{ERP}}

\newcommand{\fat}{\mathrm{fat}}

\newcommand{\err}{\operatorname{err}}
\DeclareMathOperator{\EMD}{EMD}

\begin{document}

\title{Efficient Classification for Metric Data%
\thanks{An extended abstract of this work appeared in 
Proceedings of the 23rd COLT, 2010 \cite{DBLP:conf/colt/GottliebKK10}.}
}

\author{Lee-Ad Gottlieb\thanks{
L. Gottlieb is with the 
Department of Computer Science and Mathematics at Ariel University
(email: leead@ariel.ac.il).
}\and
Aryeh Kontorovich\thanks{
A. Kontorovich is with the
Department of Computer Science at Ben-Gurion University of the Negev
(email: karyeh@cs.bgu.ac.il).
His work was supported in part by 
the Israel Science Foundation (grant No. 1141/12)
and 
a Yahoo Faculty award.
} \and 
Robert Krauthgamer\thanks{
R. Krauthgamer is with the
Faculty of Mathematics and Computer Science
at
the Weizmann Institute of Science
(email: robert.krauthgamer@weizmann.ac.il).
His work was supported in part by 
the Israel Science Foundation (grant \#452/08), the US-Israel BSF (grant \#2010418), and by a Minerva grant.
}
}

\maketitle

\begin{abstract}
Recent advances in large-margin classification of data residing in general metric
spaces (rather than Hilbert spaces) enable classification under
various natural metrics, such as string edit and earthmover distance.
A general framework developed for this purpose by 
von Luxburg and Bousquet [JMLR, 2004]
left open the questions of computational efficiency
and of providing direct bounds on generalization error.

We design a new algorithm for classification in general metric spaces,
whose runtime and accuracy depend on the doubling dimension of the data points,
and can thus achieve superior classification performance in many common scenarios.
The algorithmic core of our approach is an approximate (rather than exact) solution
to the classical problems of Lipschitz extension and of Nearest Neighbor Search.
The algorithm's generalization performance is guaranteed 
via the fat-shattering dimension of Lipschitz classifiers,
and we present experimental evidence of its superiority to some common kernel methods.
As a by-product, we offer a new perspective on the nearest neighbor classifier, which
yields significantly sharper risk asymptotics than the classic analysis of 
Cover and Hart [IEEE Trans. Info. Theory, 1967]. 
\end{abstract}

\section{Introduction}
A recent line of work extends
the large-margin classification
paradigm from Hilbert spaces to less structured ones, 
such as Banach or even metric spaces,
see e.g.~\citep{DBLP:journals/jcss/HeinBS05,DBLP:journals/jmlr/LuxburgB04,der-lee-banach,DBLP:journals/jmlr/ZhangXZ09}.
In this metric approach, data is presented as points with distances
but lacking the additional structure of inner-products.
The potentially significant advantage is that the metric can be precisely suited 
to the type of data, 
e.g.\ earthmover distance for images, or edit distance for sequences.

However, much of the existing machinery of classification algorithms and generalization bounds,
(e.g.~\citep{DBLP:journals/ml/CortesV95,smola-kernels-book}) %
depends strongly on the data residing in a Hilbert space. %
This structural requirement severely limits this machinery's applicability --- 
many natural metric spaces cannot be represented in a Hilbert space faithfully;
formally, every embedding into a Hilbert space 
of metrics such as $\ell_1$, earthmover, and edit distance 
must distort distances by a large factor \citep{Enflo69,NS07,AK10}.
Ad-hoc solutions such as kernelization cannot circumvent this shortcoming, 
because imposing an inner-product obviously embeds the data 
in some Hilbert space.

To address this gap,
von Luxburg and Bousquet 
\cite{DBLP:journals/jmlr/LuxburgB04}
developed a powerful framework of large-margin
classification for a general metric space $\X$.
They first show that the natural hypotheses (classifiers) to consider 
in this context are maximally-smooth Lipschitz functions;
indeed, they reduce classification (of points in a metric space $\X$)
with no training error
to finding a Lipschitz function $f:\X\to\R$ consistent with the data,
which is a classic problem in Analysis, known as {\em Lipschitz extension}.
Next, they establish error bounds in the form of expected surrogate loss.
Finally, the computational problem of evaluating the classification function
is reduced, assuming zero training error, to exact nearest neighbor search.
This matches a popular classification heuristic, 
and in retrospect provides a rigorous explanation 
for this heuristic's empirical success in general metric spaces,
extending the seminal analysis of Cover and Hart \cite{CoverHart67}
for the Euclidean case.

The work of \cite{DBLP:journals/jmlr/LuxburgB04} 
has left open some algorithmic questions.
In particular, 
allowing nonzero training error is apt to significantly reduce the
Lipschitz constant, thereby producing classifiers that have lower complexity
and are less likely to overfit.
This introduces the algorithmic challenge
of constructing a Lipschitz classifier that minimizes the 0-1 training error.
In addition, exact nearest neighbor search in general metrics
has time complexity proportional to the size of the dataset,
rendering the technique impractical when the training sample is large.
Finally, bounds on the expected surrogate loss may significantly overestimate the 
generalization error, which is the
true quantity of interest.

\myparagraph{Our contribution} 
We solve the problems delineated above by showing that data residing in 
a metric space of low doubling dimension 
admits accurate and computationally efficient classification.
This is the first result that ties
the doubling dimension of the \emph{data} 
to either classification error or algorithmic runtime.%
\footnote{Previously, the doubling dimension of the \emph{space of classifiers} 
was used in \cite{Bshouty2009323}, 
but this is less relevant to our discussion.}
Specifically, we 
(i) prove generalization bounds for the classification (0-1) error 
as opposed to surrogate loss,
(ii) construct and evaluate the classifier in a computationally-efficient manner, 
and (iii) perform efficient structural risk minimization by optimizing
the tradeoff between the
classifier's smoothness
and
its training error.

Our generalization bound for Lipschitz classifiers
controls 
the expected classification error directly (rather than expected surrogate loss),
and may be significantly sharper than the latter in many common scenarios.
We provide this bound in Section \ref{sec:generalization},
using an elementary analysis of the fat-shattering dimension.
In hindsight, our approach offers a new perspective on the nearest neighbor classifier, 
with significantly tighter risk asymptotics than the classic analysis of
Cover and Hart~\cite{CoverHart67}.

We further give efficient algorithms to implement the 
Lipschitz classifier, both for the training and the evaluation stages.
In Section \ref{sec:classifier} we prove that once a Lipschitz classifier 
has been chosen, 
the hypothesis can be evaluated quickly on any new point $x\in\X$
using \emph{approximate} nearest neighbor search, 
which is known to be fast when points have a low doubling dimension.
In Section \ref{sec:tradeoff} we further show how to quickly compute 
a near-optimal classifier (in terms of classification error bound),
even when the training error is nonzero. In particular, this necessitates the 
optimization of the number of incorrectly labeled
examples --- and moreover, their identity --- as part of the structural risk minimization.

Finally, we give in Section \ref{sec:example} two exemplary setups. 
In the first, the data is represented using the earthmover metric over the plane.
In the second, the data is a set of time series vectors equipped with a popular
distance function. We provide basic theoretical and experimental analysis,
which illustrate the potential power of our approach.

\section{Definitions and notation}
\label{sec:defn}

\myparagraph{Notation}
We will use standard $O(\cdot),\Omega(\cdot)$ notation for orders of magnitude.
If $f=O(g)$ and $g=O(f)$, we will write $f=\Theta(g)$.
Whenever $f=O(n \polylog n)$, we will denote this by $f=\tilde O(n)$.
If $n\in\N$ is a natural number $[n]$ denotes the set $\set{1,\ldots,n}$.

\myparagraph{Metric spaces}
A {\em metric} $\rho$ on a set $\X$ is a positive symmetric function
satisfying the triangle inequality $\rho(x,y)\leq \rho(x,z)+\rho(z,y)$; together the two comprise the metric space $(\X,\rho)$.
The diameter of a set $A\subseteq\X$, 
is defined by $\diam(A)=\sup_{x,y\in A}\rho(x,y)$
and the distance between two sets $A,B\subset\X$ is defined by
$\rho(A,B)=\inf_{x\in A,y\in B}\rho(x,y)$.
The Lipschitz constant of a function $f:\X\to\R$, denoted by $\Lip{f}$, is defined to be the smallest $L>0$ that satisfies
$\abs{f(x)-f(y)}\leq L\rho(x,y)$ for all $x,y\in\X$.

\myparagraph{Doubling dimension}
For a metric space $(\X,\rho)$, let
$\lambda$
be the smallest value such that every
ball in $\X$ can be covered by $\lambda$ balls of half the radius.
$\lambda$ is the {\em doubling constant} of $\X$,
the {\em doubling dimension} of $\X$ is $\ddim(\X)=\log_2\lambda$.
A metric is {\em doubling}
when its doubling dimension is bounded. Note that while a low Euclidean
dimension implies a low doubling dimension (Euclidean metrics of dimension
$d$ have doubling dimension $\Theta(d)$ \citep{DBLP:conf/focs/GuptaKL03}), low doubling
dimension is strictly more general than low Euclidean dimension.

The following packing property can be demonstrated via repeated applications of
the doubling property (see, for example \cite{KL04}):
\begin{lem}
\label{lem:doublpack}
Let $\X$ be a metric space, and suppose that $S\subset\X$ is finite 
and has a minimum interpoint distance
at least $\alpha>0$. Then the cardinality of $S$ is
\beq
|S| &\leq& \Big(\tfrac{2\diam(S)}{\alpha}\Big)\powddim.
\eeq
\end{lem}

\myparagraph{Nets}
Let $(\X,\rho)$ be a metric space and suppose $S\subset\X$.
An \emph{$\eps$-net} of $S$ is a subset $T \subset S$ with the following
properties:
(i) Packing: all distinct $u,v \in T$ satisfy $\rho(u,v) \ge \eps$,
which means that $T$ is $\eps$-separated; and
(ii) Covering: every point $u \in S$ is strictly within distance 
$\eps$ of some point $v \in T$, namely $\rho(y,x) < \eps$.

\myparagraph{Learning}
Our setting in this paper is the {\em agnostic PAC} learning model
\citep{mohri-book2012}.
Examples are drawn independently from
$\X\times\set{-1,1}$ according to some unknown 
probability
distribution $\P$ and
the learner, having observed $n$ such pairs $(x,y)$ produces a hypothesis
$h:\X\to\set{-1,1}$. The {\em generalization error} is the probability of
misclassifying a new point drawn from $\P$:
$$ \P\set{(x,y): h(x)\neq y}.$$
The quantity above is random, since it depends on the $n$ observations, and
we wish to
upper-bound it in probability.
Most bounds of this sort contain a {\em training error} term, 
which is the fraction of observed examples misclassified by $h$ 
and roughly corresponding to bias in Statistics,
as well as a {\em hypothesis complexity} term,
which measures the richness of the class of all admissible hypotheses \citep{MR2172729},
and roughly corresponding to variance in Statistics.
Optimizing the tradeoff between these two terms is known as Structural Risk Minimization (SRM).%
\footnote{
Robert Schapire pointed out to us that these terms from Statistics are not entirely accurate in the machine learning setting. 
In particular,
the classifier complexity term does not correspond to the variance of the classifier in any quantitatively precise way. 
However, 
the intuition underlying SRM corresponds precisely to the one behind 
bias-variance tradeoff in Statistics,
and so we shall occasionally use the latter term as well.
}
Keeping in line with the literature, we ignore the measure-theoretic technicalities associated with taking suprema over uncountable function classes.

\section{Generalization bounds}
\label{sec:generalization}

In this section, we derive generalization bounds  for Lipschitz classifiers over doubling spaces.
As noted by \cite{DBLP:journals/jmlr/LuxburgB04}
Lipschitz functions are the natural object to consider in an optimization/regularization framework.
The basic intuition behind our proofs is that the Lipschitz constant plays the role of the inverse margin in the confidence of the classifier.
As in
\cite{DBLP:journals/jmlr/LuxburgB04},
small Lipschitz constant corresponds to large margin, which in turn yields low hypothesis complexity and variance.
However, in contrast to
\cite{DBLP:journals/jmlr/LuxburgB04}
(whose generalization bounds rely on Rademacher averages) we use the doubling property of the metric space directly to
control the fat-shattering dimension.

We apply
tools from
generalized Vapnik-Chervonenkis theory to the case of Lipschitz classifiers.
Let $\F$ be a collection of functions $f:\X\to\R$
and
recall the definition of the fat-shattering dimension \citep{alon97scalesensitive,299098}: a set $X\subset\X$ is said to 
be $\gamma$-shattered
by $\F$ if there exists some function $r:X\to\R$ such that for each label assignment $y\in\set{-1,1}^X$ there is an
$f\in\F$ satisfying $y(x)(f(x)-r(x))\geq\gamma>0$
for all $x\in X$. The $\gamma$-fat-shattering dimension of $\F$, denoted by $\fat_\gamma(\F)$, is the cardinality of the largest set $\gamma$-shattered
by $\F$.

For the case of Lipschitz functions, we will show that the notion of fat-shattering dimension may be somewhat simplified.
We say that a set $X\subset\X$ is
$\gamma$-shattered {\em at zero}
by a collection of functions $\F$ if
for each $y\in\set{-1,1}^X$ there is an $f\in\F$ satisfying $y(x)f(x)\geq\gamma$
for all $x\in X$. (This is the definition above with $r\equiv0$.)
We write $\fat_\gamma^0(\F)$ to denote
the cardinality of the largest set $\gamma$-shattered at zero by $\F$ and 
show that for Lipschitz function classes the two notions are the same.

\begin{lem}
\label{lem:fat0}
Let $\F$ be the collection of all $f:\X\to\R$ with $\Lip{f}\leq L$. Then $\fat_\gamma(\F)=\fat_\gamma^0(\F)$.
\end{lem}
\begin{proof}
We begin by recalling the classic Lipschitz extension result, essentially due to 
\cite{MR1562984} and \cite{1934}.
Any real-valued function $f$ defined on a subset $X$ of a metric space $\X$ has an extension $f^*$ to all of $\X$
satisfying $\Lip{f^*}=\Lip{f}$. Thus, in what follows we will assume that any function
$f$ defined on $X\subset\X$ is also defined on all of $\X$ via some Lipschitz extension
(in particular, to bound  $\Lip{f}$ it suffices to bound the restricted $\Lip{\evalat{f}{X}}$).

Consider some finite $X\subset\X$. If $X$ is
$\gamma$-shattered at zero by $\F$ then
by definition
it is also $\gamma$-shattered.
Now assume that $X$ is $\gamma$-shattered by $\F$.
Thus,
there is some function
$r:X\to\R$ such that for each $y\in\set{-1,1}^X$ there is an $f=f_{r,y}\in\F$
such that $f_{r,y}(x)\geq r(x)+\gamma$ if $y(x)=+1$ and
$f_{r,y}(x)\leq r(x)-\gamma$ if $y(x)=-1$.
Let us define the function $\tilde f_y$ on $X$ and as per above, on all of $\X$,
by $\tilde f_y(x)=\gamma y(x)$. It is clear that
the collection $\set{\tilde f_y:y\in\set{-1,1}^X}$  $\gamma$-fat-shatters $X$ at zero; it only
remains to verify that $\tilde f_y\in\F$, i.e.,
\beq
\sup_{y\in\set{-1,1}^X} \Lip{\tilde f_{y}}
&\leq&
\sup_{y\in\set{-1,1}^X} \Lip{f_{r,y}}
.
\eeq
Indeed,
\balig
\sup_{y\in\set{-1,1}^X,x,x'\in X}
\frac{f_{r,y}(x)-f_{r,y}(x')}{\rho(x,x')}
&\geq&
\sup_{x,x'\in X}
\frac{r(x)-r(x')+2\gamma}{\rho(x,x')}
\\
&\geq&
\sup_{x,x'\in X}
\frac{2\gamma}{\rho(x,x')}
\\
&=&
\sup_{y\in\set{-1,1}^X} \Lip{\tilde f_{y}}.
\ealig
\end{proof}

A consequence of Lemma \ref{lem:fat0} is that
in considering the generalization properties of Lipschitz functions
we need only
bound the $\gamma$-fat-shattering dimension at zero.
The latter is achieved by observing 
that the packing number of a metric space controls
the fat-shattering dimension of Lipschitz functions defined over the metric space:
\begin{thm}
\label{thm:fat-pack}
Let $(\X,\rho)$ be a metric space. Fix some $L>0$, and let $\F$ be the collection of all $f:\X\to\R$ with $\Lip{f}\leq L$.
Then
for all $\gamma>0$,
\beq
\fat_\gamma(\F) =
\fat^0_\gamma(\F)
\leq
M(\X,\rho,2\gamma/L)
\eeq
where $M(\X,\rho,\eps)$ is the $\eps$-packing number of $\X$,
defined as
the cardinality of the largest $\eps$-separated subset of $\X$.
\end{thm}
\begin{proof}
Suppose that $S\subseteq\X$ is fat $\gamma$-shattered at zero.
The case $|S|=1$ is trivial, so we assume the existence of $x\neq x'\in S$ and $f\in\F$ such that $f(x)\geq\gamma>-\gamma\geq f(x')$. The Lipschitz property
then implies that $\rho(x,x')\geq 2\gamma/L$, and the claim follows.
\end{proof}

\begin{cor}
\label{cor:doubling}
Let metric space $\X$ have 
doubling dimension
$\ddim(\X)$, and let $\F$ be the collection of
real-valued functions over $\X$ with Lipschitz constant
at most
$L$. Then
for all $\gamma>0$,
\beq
\fat_\gamma(\F) \leq \paren{\frac{L\,\diam(\X)}{\gamma}}\powddim.
\eeq

\end{cor}
\begin{proof}
The claim follows immediately from Theorem \ref{thm:fat-pack} and the packing property of doubling spaces (Lemma \ref{lem:doublpack}).
\end{proof}

Equipped with these estimates for
the fat-shattering dimension
of Lipschitz classifiers,
we can invoke a standard generalization bound stated in terms of this quantity.
For the remainder of this section,
we take $\gamma=1$ and say that a function $f$ classifies an
example $(x_i,y_i)$ correctly if
\beqn
\label{eq:yf1}
 y_i f(x_i)\geq1.
\eeqn
The following generalization bounds appear in \cite{299098}.

\begin{thm}\label{thm:gen}
Let $\F$ be a collection of real-valued functions over some set $\X$,
define $D=\fat_{1/16}(\F)$ and let
and $P$ be some probability distribution on $\X\times\set{-1,1}$.
Suppose that $(x_i,y_i)$, $i=1,\ldots,n$ are drawn from $\X\times\set{-1,1}$ independently according to $P$ and that some
$ f\in\F$ 
classifies the $n$ examples correctly, in the sense of (\ref{eq:yf1}).
Then
with probability at least $1-\delta$,
\balig
\P\set{ (x,y) : \sgn( f(x))\neq y}
&\leq &
\frac{2}{n}\paren{D\log_2(34en/D)\log_2(578n)+\log_2(4/\delta)}.
\ealig
Furthermore, if $ f\in\F$ 
is correct on
all but $k$
examples, we have
with probability at least $1-\delta$
\balig
\P\set{ (x,y) : \sgn( f(x))\neq y}
&\leq&
\frac{k}{n}+\sqrt{
\frac{2}{n}\paren{D\ln(34en/D)\log_2(578n)+\ln(4/\delta)}
}
.
\ealig
\end{thm}
Applying Corollary \ref{cor:doubling}, we obtain the following consequence of
Theorem \ref{thm:gen}.

\begin{cor}
\label{cor:conseq}
Let metric space $\X$ have 
doubling dimension
$\ddim(\X)$, and let $\F$ be the collection of
real-valued functions over $\X$ with Lipschitz constant
at most
$L$.
Then
for any $f\in\F$ that 
classifies a sample of size $n$ correctly,
we have
with probability at least $1-\delta$
\balign
\label{eq:ourbd-sep}
\P\set{ (x,y) : \sgn( f(x))\neq y}
&\leq&
 \frac{2}{n}\paren{D\log_2(34en/D)\log_2(578n)+\log_2(4/\delta)}.
\ealign
Likewise, if $f$ 
is correct
on 
all but $k$
examples, we have
with probability at least $1-\delta$
\balign
\label{eq:ourbd}
\P\set{ (x,y) : \sgn( f(x))\neq y}
&\leq&
\frac{k}{n}+\sqrt{
\frac{2}{n}\paren{D\ln(34en/D)\log_2(578n)+\ln(4/\delta)}
}
.
\ealign
In both cases, $D=\fat_{
1
/16}(\F)
\leq
\paren{16L\,\diam(\X)}\powddim
$.
\end{cor}

\subsection{Comparison with previous generalization bounds}
Our generalization bounds are not directly comparable to those of
von Luxburg and Bousquet \cite{DBLP:journals/jmlr/LuxburgB04}.
In general, 
two approaches exist 
to analyze binary classification by continuous-valued functions:
thresholding by the sign function or bounding some expected surrogate loss function. 
They opt for the latter approach, defining the surrogate loss function 
$$\ell(f(x),y)= \min(\max(0,1-yf(x)),1)$$ 
and bound the risk $\E[\ell(f(x),y)]$.
We take the former approach, bounding the generalization error
$\P(\sgn(f(x))\neq y)$ directly. Although for $\set{-1,1}$-valued labels
the risk upper-bounds the generalization error,
it could potentially be a crude overestimate.

von Luxburg and Bousquet \cite{DBLP:journals/jmlr/LuxburgB04}
demonstrated that the Rademacher average of Lipschitz functions over
the $p$-dimensional unit cube ($p\geq 3$) is of order $\Theta(n^{-1/p})$,
and since the proof uses only covering numbers,
a similar bound holds for all metric spaces with bounded diameter
and doubling dimension.
In conjunction with Theorem 5(b) of \cite{DBLP:journals/jmlr/BartlettM02},
this observation yields the following bound.
\begin{lem}
\label{lem:vlB}
Let $\X$ be a metric space with $\diam(\X)\leq 1$,
and let $\F$ be the collection of all $f:\X\to\R$ with $\Lip{f}\leq1$.
If $(x_i,y_i)\in\X\times\set{-1,1}$ are drawn iid with respect to some probability distribution $P$,
then
with probability
at least $1-\delta$
every $f\in\F$ satisfies
\balig
\P\set{(x,y):f(x)\neq y} &\leq&
O\paren{k_f/n + n^{-1/\ddim(\X)}+\sqrt{\ln(1/\delta)/n}},
\ealig
where $k_f$ is the number of examples $f$ labels incorrectly.
\end{lem}
Our results compare favorably to those of \cite{DBLP:journals/jmlr/LuxburgB04}
when we assume fixed diameter $\diam(\X)\leq 1$ and Lipschitz constant $L\leq 1$
and the number of observations $n$ goes to infinity.
Indeed, Lemma \ref{lem:vlB} bounds the excess error decay by $O(n^{-1/\ddim(\X)})$, 
whereas Corollary \ref{cor:conseq} gives a rate of $\tilde O(n^{-1/2})$.

\subsection{Comparison with previous nearest-neighbor bounds}
Corollary \ref{cor:conseq} also allows us to significantly sharpen the asymptotic analysis of 
\cite{CoverHart67} for the nearest-neighbor classifier. Following the presentation in 
\cite{shwartz2014understanding}
with an appropriate generalization to general metric spaces, the analysis
of \cite{CoverHart67} implies
that the $1$-nearest-neighbor classifier $h_{\mathrm{NN}}$ achieves 
\balign
\label{eq:CH}
\E[\err(
h_{\mathrm{NN}}
)] 
&\leq&
2\err(h^*) 
+ O\paren{ \frac{\Lip{\eta}\diam(\X)}{n^{1/({\ddim(\X)+1})}}},
\qquad 
\ealign
where $\eta(x)=\P(Y=1\gn X=x)$
is the conditional probability of the $1$ label, and $h^*(x)=\sgn(\eta(x)-1/2)$
is the Bayes optimal classifier.
The curse of dimensionality exhibited in the term 
$n^{1/({\ddim(\X)+1})}$ is real --- 
for each $n$, there exists a distribution such that for sample size
$n\ll (L+1)\powddim$,
we have 
$\E[\err(h_{\mathrm{NN}})] \geq \Omega(1)$.
However, Corollary \ref{cor:conseq} shows that this analysis is overly pessimistic.
Comparing (\ref{eq:CH}) with (\ref{eq:ourbd-sep}) in the case where $\err(h^*)=0$,
we see that once the sample size passes 
a critical number on the order of $(L\,\diam(\X))\powddim$,
the expected generalization error begins to decay as $\tilde O(1/n)$, 
which is much faster than the rate suggested by (\ref{eq:CH}).

\hide{
\begin{figure}[t]
\begin{center}
\scalebox{.5}{
\includegraphics{ch-gkk}
}
\caption{Comparison between the error estimate from (\ref{eq:CH}) ---
with the constants taken from \cite{shwartz2014understanding} ---
and our bound in (\ref{eq:ourbd}), 
denoted as the ``CH'' and ``GKK'' curves, respectively.
The plot corresponds to $\diam=1$, $\ddim=10$, and $L=1$,
assuming consistent classification on the training sample. 
Our bound requires a critical sample size (for the given settings, 
around $10^{12}$) before giving a nontrivial (i.e., $<1$) error bound, 
but once that size is reached, our bound decays exponentially faster.
}
\label{fig:ch-gkk}
\end{center}
\end{figure}
}

\section{Lipschitz extension classifier}
\label{sec:classifier}

Given $n$ labeled points $(x_1,y_1),\ldots,(x_n,y_n)\in\X\times\set{-1,1}$,
we construct our classifier in a similar manner to
\cite{DBLP:journals/jmlr/LuxburgB04},
via a Lipschitz extension of the label values $y_i$ to all of $\X$.
Let $S^+,S^-\subset \set{x_1,\ldots,x_n}$ be the sets of positive and negative 
labeled points. 
Our starting point is the same extension function used in
\cite{DBLP:journals/jmlr/LuxburgB04},
namely, for $\alpha \in [0,1]$ define $f_\alpha:\X\to\R$ by
\balign
\label{eq:f-alpha}
f_\alpha (x)
= \alpha \min_{i\in[n]} \left( y_i + 2\frac{\rho(x,x_i)}{\rho(S^+,S^-)} \right)
+ (1-\alpha) \max_{j\in[n]} \left( y_j - 2\frac{\rho(x,x_j)}{\rho(S^+,S^-)} \right).
\ealign
It is easy to verify,
see also \cite[Lemmas 7 and 12]{DBLP:journals/jmlr/LuxburgB04},
that $f_\alpha(x_i)$ agrees with the sample label $y_i$ for all $i\in[n]$,
and that its Lipschitz constant 
is identical to the one induced by the labeled points,
which in turn is obviously $2/\rho(S^+,S^-)$.
However, computing the exact value of $f_\alpha(x)$ for a point $x\in \X$ 
(or even the sign of $f_\alpha(x)$ at this point) 
requires an exact nearest neighbor search, and in arbitrary 
metric space nearest neighbor search requires $\Omega(n)$ time. 

In this section, we design a classifier that is evaluated at a point $x\in\X$ 
using an approximate nearest neighbor search.%
\footnote{If $x^*$ is the nearest neighbor for a test point $x$,
then any point $\tilde x$ satisfying $\rho(x,\tilde x)\le(1+\eps)\rho(x,x^*)$
is called a $(1+\eps)$-approximate nearest neighbor of $x$.
} 
It is known how to build a data structure for a set of $n$ points
in time $2^{O(\ddim(\X))}n \log n$,
so as to support $(1+\eps)$-approximate nearest neighbor searches in time 
$2^{O(\ddim(\X))}\log n + \eps^{-O(\ddim(\X))}$
\citep{DBLP:conf/stoc/ColeG06,1122817} (see also \cite{KL04,1143857}).
Our classifier below relies only on a given subset of the given $n$ points,
which may eventually lead to improved generalization bounds
(i.e., it provides a tradeoff between $k$ and $L$ in Theorem~\ref{thm:gen}).

\begin{thm}
\label{thm:LE-class}
Let $(\X,\rho)$ be a metric space, and fix $0 < \eps < \frac{1}{32}$.
Let $S$ be a sample consisting of $n$ labeled points 
$(x_1,y_1),\ldots,(x_n,y_n)\in\X\times\set{-1,1}$.
Fix a subset $S_1\subset S$ of cardinality $n-k$,
on which the constructed classifier must agree with the given labels,
and partition it into $S_1^+,S_1^-\subset S_1$ according to the labels,
letting $L=2/\rho(S_1^+,S_1^-)$.
Then there is a binary classification function $h:\X\to\set{-1,1}$ satisfying:
\bit
\item[(a)]
$h(x)$ can be evaluated at each $x\in\X$ in
time $2^{O(\ddim(\X))}\log n + \eps^{-O(\ddim(\X))}$, after an initial computation of
$(2^{O(\ddim(\X))} \log n + \eps^{-O(\ddim(\X))})n$ time.
\item[(b)]
With probability at least $1-\delta$ (over the sampling of $S$)
\balig
\P\set{ (x,y) : h(x)\neq y}
&\leq&
{\frac{k}{n}+\sqrt{
\frac{2}{n}\paren{D\ln(34en/D)\log_2(578n)+\ln(4/\delta)}
}},
\ealig
where 
$$D = \paren{
\frac{16L\,\diam(\X)}{1-32\eps}
}\powddim.$$
\eit
\end{thm}

We will use the following simple lemma.
\begin{lem}
\label{lem:fateps}
For any function class $\calF$ mapping $\X$ to $\R$, 
define its $\eps$-perturbation $\F_\eps$ to be
\beq
\F_\eps=\set{\ftil\in\R^\X : \tsnrm{f-\ftil}_\infty \le\eps, f\in\F},
\eeq
where $\tsnrm{f-\ftil}_\infty =\sup_{x\in\X}\tsabs{f(x)-\ftil(x)}$.
Then for $0<\eps<\gamma$,
\beq
\fat_\gamma(\F_\eps) \le \fat_{\gamma-\eps}(\F).
\eeq
\end{lem}

\begin{proof}
Suppose that $\F_\eps$ is able to $\gamma$-shatter the 
finite subset $X\subset\X$.
Then there is an $r\in\R^X$ so that for all
$y\in\set{-1,1}^X$,
there is an
$\tilde f_y\in \F_\eps$ such that
\beqn
\label{eq:gameps}
y(x)(\tilde f_y(x)-r(x))\geq\gamma,
\qquad \forall x\in X.
\eeqn
Now by definition, for each $\tilde f_y\in\F_\eps$ there is 
some $f_y\in\F$ such that $\sup_{x\in\X}|f_y(x)-\tilde f_y(x)|\leq\eps$.
We claim that the collection
$\set{f_y: y\in\set{-1,1}^X}$ 
is able to
$(\gamma-\eps)$-shatter $X$.
Indeed, replacing
$\tilde f_y(x)$ with $f_y(x)$ in 
(\ref{eq:gameps}) perturbs the left-hand side by an additive term of at most $\eps$.
\end{proof}

\begin{proof}[Proof of Theorem~\ref{thm:LE-class}]
Without loss of generality, assume $S_1\subset S$ corresponds to points
indexed by $i=1,\ldots,n-k$.
We begin by observing that since all of the sample labels have values in $\set{\pm1}$,
any Lipschitz extension may be truncated to the range $[-1,1]$. 
Formally, if $g$ is a Lipschitz extension of the labels $y_i$
from the sample $S$ to all of $\X$, 
then so is $\trun{-1}{1}\circ g$, where
\beq
\trun{a}{b}(z) = \max\set{a,\min\set{b,z}}
\eeq
is the truncation operator. 
In particular, take $g$ to be as in \eqref{eq:f-alpha} with $\alpha = 1$
and write
\beq
r_i(x) &=& 2\frac{\rho(x,x_i)}{\rho(S_1^+,S_1^-)}.
\eeq
Now defining
\balign
\label{eq:lipext} 
f(x) 
  = \trun{-1}{1} \left( \min_{i\in[n-k]} \set{y_i + r_i(x)}  \right) 
  = \min_{i\in[n-k]}  \left\{ \trun{-1}{1}(y_i + r_i(x)) \right\},
\ealign
where the second equality is by monotonicity of the truncation operator,
we conclude that $f$ is a Lipschitz extension of the data,
with the same Lipschitz constant $L=2/\rho(S_1^+,S_1^-)$.

Now precompute%
\footnote{The word {\em precompute} underscores the fact that this computation
is done during the ``offline'' learning phase.
Its result is then used to achieve fast ``online'' evaluation of the classifier
on any point $x\in\X$ during the testing phase.
} 
in time $2^{O(\ddim(\X))} n \log n$ 
a data structure that supports $(1+\eps)$-approximate nearest neighbor 
searches on the point set $S_1^+$,
and a similar one for the point set $S_1^-$.
Now compute (still during the learning phase) 
an estimate $\tilde\rho(S_1^+,S_1^-)$ for $\rho(S_1^+,S_1^-)$,
by searching the second data structure for each 
of the points
in $S_1^+$,
and taking the minimum of all the resulting distances.
This estimate satisfies
\beqn \label{eq:tilrho}
  1 \leq \frac{\tilde\rho(S_1^+,S_1^-)}{\rho(S_1^+,S_1^-)} \leq 1+\eps,
\eeqn
and this entire precomputation process takes 
$(2^{O(\ddim(\X))} \log n + \eps^{-O(\ddim(\X))})n$ time.

Given a test point $x\in\X$ to be classified, 
search for $x$ in the two data structures (for $S_1^+$ and for $S_1^-$), 
and denote the indices of the points answered by them
by $a^+,a^-\in[n-k]$, respectively. 
The $(1+\eps)$-approximation guarantee means that
\beq 
  1\le \frac {\rho(x,x_{a^+})}{\rho(x,S_1^+)} \leq 1+\eps,
  \mbox{ and } \quad
  1\le \frac {\rho(x,x_{a^-})}{\rho(x,S_1^-)} \leq 1+\eps.
\eeq
Define, as a computationally-efficient estimate of $f$, the function
\beq 
\ftil(x) 
  =  \min_{a\in \set{a^+,a^-}} \set{ 
        \trun{-1}{1} \left( 
           y_a + 2\frac{\rho(x,x_a)}{\tilde\rho(S_1^+,S_1^-)} 
         \right)
       } ,
\eeq
and let our classifier be $h(x)=\sgn(\ftil(x))$.
We remark that the case $a=a^-$ always attains the minimum 
in the definition of $\ftil$ 
(because $a=a^+$ only produces values greater or equal than $y_{a^+}=1$),
and therefore one can avoid the computation of $a^+$,
and even the construction of a data structure for $S_1^+$.
In fact, the same argument shows that also in the definition 
of $f$ in \eqref{eq:lipext} 
we can omit from the minimization points with label $y_i=+1$.

This classifier $h=\sgn(\ftil)$ can be evaluated 
on a new point $x\in\X$ in time $2^{O(\ddim(\X))} \log n + \eps^{-O(\ddim(\X))}$,
and it thus remains to bound the generalization error of $h$. 
To this end, we will show that
\beqn
\label{eq:fftil}
\sup_{x\in\X}\abs{f(x)-\ftil(x)} \le 2\eps.
\eeqn
This means that $\ftil$ is a $2\eps$-perturbation of $f$, 
as stipulated by Lemma~\ref{lem:fateps}, 
and the generalization error of $h$ will follow from Corollary~\ref{cor:conseq}
using the fact that $f$ has Lipschitz constant $L$.

To prove (\ref{eq:fftil}), fix an $x\in\X$.
Now let $i^*\in[n-k]$ be an index attaining the minimum in the definition 
of $f(x)$ in \eqref{eq:lipext},
and similarly $a^*\in[n-k]$ for $\ftil(x)$.
Using the remark above, we may assume that their labels are $y_{i^*}=y_{a^*}=-1$.
Moreover, by inspecting the definition of $f$ we may further assume that
$i^*$ attains the minimum of $r_{i}(x)$ (over all points labeled $-1$) 
and thus also of its numerator $\rho(x,x_{i})$.
And since index $a^*$ was chosen as an approximate nearest neighbor 
(among all points labeled $-1$), we get
$1 \leq \frac{\rho(x,x_{a^*})}{\rho(x,x_{i^*})} \leq 1+\eps$.
Together with \eqref{eq:tilrho}, we have
\beqn \label{eq:ri}
  \frac{1}{1+\eps}\ r_{i^*}(x)
  \leq 2\frac{\rho(x,x_{a^*})}{\tilde\rho(S_1^+,S_1^-)} 
  \leq (1+\eps)\ r_{i^*}(x).
\eeqn
We now need the following simple claim:
\balig
  0\leq B \leq (1+\eps)C
  \quad \Longrightarrow 
  \quad \trun{-1}{1}(-1+B) \leq 2\eps + \trun{-1}{1}(-1+C).
\ealig
To verify the claim, assume first that $C \leq 2$;
then $B \leq C + \eps C \leq C + 2\eps$,
and now use the fact that adding $-1$ and truncating 
are both monotone operations, to get
$ \trun{-1}{1}(-1+B) \leq \trun{-1}{1}(-1+C + 2\eps)$,
and the right-hand side is clearly at most $\trun{-1}{1}(-1+C) + 2\eps$.
Assume next that $C\geq 2$; 
then obviously $\trun{-1}{1}(-1+B) \leq 1 = \trun{-1}{1}(-1+C)$.
The claim follows.

Applying this simple claim twice, once for each inequality in \eqref{eq:ri},
we obtain that 
\balig
      -2\eps
      \leq 
      \trun{-1}{1} \Big(y_{i^*} + r_{i^*}(x) \Big) 
      -
      \trun{-1}{1} \left( 
           y_{a^*} + 2\frac{\rho(x,x_{a^*})}{\tilde\rho(S_1^+,S_1^-)} 
         \right)
      \leq 
      2\eps,
\ealig
which proves \eqref{eq:fftil}, and completes the proof of the theorem.
\end{proof}

\section{Structural Risk Minimization}\label{sec:tradeoff}

In this section, we show how to efficiently construct a classifier that optimizes the ``bias-variance
tradeoff'' implicit in Corollary \ref{cor:conseq}, equation (\ref{eq:ourbd}). Let $\X$ be a metric 
space, and assume we are given a labeled sample $S=(x_i,y_i)\in\X\times\set{-1,1}$. For any Lipschitz 
constant $L$, let $k(L)$ be the minimal training error of $S$ over all classifiers with Lipschitz 
constant $L$. We rewrite the generalization bound as follows:
\balign
\label{eq:GL}
\P\set{ (x,y) : \sgn( f(x))\neq y}
&\leq&
\frac{k(L)}{n}+\sqrt{
\frac{2}{n}\paren{D\ln(34en/D)\log_2(578n)+\ln(4/\delta)}
}
=:G(L) 
\ealign
where $D = \paren{16L\,\diam(\X)}\powddim$. This bound contains a free parameter, 
$L$, which may be tuned 
in the course of structural risk minimization.
More precisely, decreasing $L$ 
drives the ``bias'' term (number of mistakes) up and the ``variance'' term (fat-shattering dimension)  down.
We thus seek an (optimal) value of $L$ where
$G(L)$ achieves a minimum value,
as described in the following theorem, which is our SRM result. 

\begin{thm}
Let $\X$ be a metric space and $0<\eps<\frac1{32}$. 
Given a labeled sample $S=(x_i,y_i)\in\X\times\set{-1,1}$, 
$i=1,\ldots,n$, there exists a binary classification function $h:\X\to\set{-1,1}$ satisfying the 
following properties:
\bit
\item[(a)]
$h(x)$ can be evaluated at each $x\in\X$ in
time $2^{O(\ddim(\X))}\log n + \eps^{-O(\ddim(\X))}$, after an initial computation of
$2^{O(\ddim(\X))} n \log n + \left( \frac{\ddim(\X)}{\eps} \right)^{O(\ddim(\X))}n$
time.
\item[(b)]
The generalization error of $h$ is bounded by
\balig
\P\set{ (x,y) : \sgn( f(x))\neq y}
\leq
c \cdot \inf_{L>0} \left[ \frac{k(L)}{n}+ \right.
\left. \sqrt{
\frac{2}{n}\paren{D\ln(34en/D)\log_2(578n)+\ln\frac4\delta}
}
\right]
.
\ealig
for some constant $c \le 2(1+\eps)$, and where $$D=D(L) = 
\paren{
\frac{16L\,\diam(\X)}{1-32\eps}
}\powddim
.$$
\eit
\end{thm}

We proceed with a description of our algorithm. 
We first give an algorithm with runtime $O(n^{4.373})$,
and then improve the runtime, first to 
$O\left( \frac{\ddim(\X)}{\eps}n^{2.373} \log n \right)$, then to 
$O( \ddim(\X) n^2 \log n )$, and finally to
$2^{O(\ddim(\X))} n \log n + \left( \frac{\ddim(\X)}{\eps} \right)^{O(\ddim(\X))} n$.

\myparagraph{Algorithm description}
We start by giving a randomized algorithm that finds a value $L^*$ 
that is optimal, namely, $G(L^*) = \inf_{L>0} G(L)$
for $G(L)$ that was defined in (\ref{eq:GL}).
The runtime of this algorithm is $O(n^{4.373})$ with high probability. 
First note the behavior of $k(L)$ as $L$ increases. 
$k(L)$ decreases only when the value of $L$ crosses certain critical values,
each determined by a pair $x_i \in S^+,x_j \in S^-$ 
(that is, $L = \frac{2}{\rho(x_i,x_j)}$); for such $L$, the 
classification function $h$ can correctly classify both these points. 
There are $O(n^2)$ 
critical values of $L$, and these can be determined by enumerating all interpoint 
distances between subsets $S^+,S^- \subset S$.

Below, we will show that for any given $L$, the value $k(L)$ can be computed in 
randomized time $O(n^{2.373})$. More precisely, we will show how to compute a partition 
of $S$ into sets $S_1$ (with Lipschitz constant $L$) and $S_0$ (of size $k(L)$) in this 
time. Given sets $S_0,S_1 \subset S$, we can construct the classifier of Corollary 
\ref{cor:conseq}. Since there are $O(n^2)$ critical values of $L$, we can calculate 
$k(L)$ for all critical values in $O(n^{4.373})$ total time, and thereby determine $L^*$. 
Then by Corollary \ref{cor:conseq}, we may compute a classifier with a bias-variance 
tradeoff arbitrarily close to optimal.

To compute $k(L)$ for a given $L$ in randomized time $O(n^{2.373})$,
consider the following algorithm: 
Construct a bipartite graph $G=(V^+,V^-,E)$. 
The vertex sets $V^+,V^-$ correspond to the labeled sets $S^+,S^- \in 
S$, respectively. The length of edge $e=(u,v)$ connecting vertices $u \in V^+$ and $v \in 
V^-$ is equal to the distance between the points, and $E$ includes all edges 
of length less than $\frac{2}{L}$. (This $E$ can be computed in $O(n^2)$ time.) 
Now, for all edges $e \in E$, 
a classifier with Lipschitz constant $L$ necessarily misclassifies at least one 
endpoint of $e$. Hence, finding a classifier with Lipschitz constant $L$ 
that misclassifies a minimum number of points in $S$ is exactly 
the problem of finding a minimum vertex cover for the bipartite graph $G$. 
(This is an unweighted graph -- the lengths are used only to determine $E$.)
By K\"onig's theorem, the minimum vertex cover problem in bipartite graphs 
is equivalent to the maximum matching problem,
and a maximum matching in bipartite graphs can be computed in randomized time 
$O(n^{2.373})$ \citep{1033180, W12}.
This maximum matching immediately identifies a minimum vertex cover,
which in turn gives the subsets $S_0,S_1$, 
allowing us to compute a classifier achieving nearly optimal SRM.

\myparagraph{First improvement}
The runtime given above can be reduced from randomized $O(n^{4.373})$ to randomized
$O( \ddim(\X)/\eps\cdot n^{2.373} \log n )$, 
if we are willing to settle for a generalization bound $G(L)$ within a 
$(1+\eps)$ factor of the optimal $G(L^*)$, for any $\eps\in(0,1)$.
To achieve this improvement, we discretize the candidate values of $L$, 
and evaluate $k(L)$ only for $O ({\ddim(\X)}/{\eps}\cdot \log n )$ 
values of $L$, rather than all $\Theta(n^2)$ values as above. 
In the extreme case where the optimal hypothesis fails on all points of 
a single label, the classifier $h$ is a constant function and $L^*=0$. 
In all other cases, $L^*$ must take
values in the range $\left[ \frac{2}{\diam(\X)},\frac{n}{\diam(\X)} \right]$;
indeed, every hypothesis correctly classifying a pair of opposite labelled 
points has Lipschitz constant at least $\frac{2}{\diam(\X)}$, 
and if $L^* > \frac{n}{\diam(\X)}$ then the complexity term (and $G(L^*)$) 
is greater than $1$.

Our algorithms evaluates $k(L)$ for values of 
$L=\frac{2}{\diam(\X)} \left(1+\frac{\eps}{\ddim(\X)} \right)^i$ 
for $i=0,1,\ldots,\ceil{\log_{1+{\eps}/{\ddim(\X)}} \frac{n}{2} }$, 
and uses the candidate that minimizes $G(L)$. 
The number of candidate values for $L$ is
$O ( {\ddim(\X)}/{\eps}\cdot \log n )$,
and one of these values --- call it $L'$ --- 
satisfies $L^* \le L' < (1+\frac{\eps}{\ddim(\X)})L^*$. 
Observe that $k(L') \le k(L^*)$ and that
the complexity term for $L'$ is greater than that for $L^*$ by at most a factor
$\sqrt{ \left( 1+\frac{\eps}{\ddim(\X)} \right)\powddim} 
\le e^{\eps/2} \le 1+\eps$ (where the final inequality holds since $\eps < 1$). 
It follows that $G(L') < (1+\eps) G(L^*)$, implying that this algorithm
achieves a $(1+\eps)$-approximation to $G(L^*)$.

\myparagraph{Second improvement}
The runtime can be further reduced from randomized 
$O ( {\ddim(\X)}/{\eps}\cdot n^{2.373} \log n )$ to deterministic 
$O ( \ddim(\X) n^2 \log n )$, 
if we are willing to settle for a generalization bound $G(L)$ 
within a constant factor $2$ of the optimal $G(L^*)$.
The improvement comes from a faster vertex-cover computation. 
It is well known that a $2$-approximation to vertex cover can be computed 
(in arbitrary graphs) by a greedy algorithm in time linear in the graph size
$O(|V^+ \cup V^-|+|E|) = O(n^2)$, see e.g.~\citep{BE81}.
Hence, we can compute in $O(n^2)$ time a function $k'(L)$ that satisfies 
$k(L) \le k'(L) \le 2k(L)$.
We replace the randomized $O(n^{2.373})$ algorithm with this $O(n^2)$ time 
greedy algorithm. Then $k'(L) \le 2k(L)$, 
and because we can approximate the complexity term to a factor smaller than $2$ 
(as above, by choosing a constant $\eps <1$), our resulting algorithm 
finds a Lipschitz constant $L'$ for which $G(L') \le 2\cdot G(L^*)$.

\myparagraph{Final improvement}
We can further improve the runtime from $O(\ddim(\X) n^2 \log n)$ to 
$2^{O(\ddim(\X))} n \log n + \left( {\ddim(\X)}/{\eps} \right)^{O(\ddim(\X))} n$,
at the cost of increasing the approximation factor to $2(1+\eps)$. 
The idea is to work with a sparser representation of the vertex cover problem.
Recall that we discretized the values of $L$ to powers of $\left( 1 + \frac{\eps}{\ddim(\X)} \right)$.
As was already observed by \cite{KL04, DBLP:conf/stoc/ColeG06} in the context of hierarchies for doubling metric, $\X$ contains at most $\eps^{-1} 2^{O(\ddim(\X))} n$ of these distinct rounded critical values.
After constructing a standard hierarchy (in time $2^{O(\ddim(\X))} n \log n$), these ordered values
may be extracted with $\left( \frac{\ddim(\X)}{\eps} \right)^{O(\ddim(\X))} n$ more work.

Let $L$ be a discretized value considered above. We extract from $S$ a subset $S' \subset S$ 
that is a $\left( \frac{\eps}{\ddim(\X)} \cdot \frac{2}{L} \right)$-net for $S$. Map 
each point $p \in S$ to its closest net point $p' \in S'$, and maintain for each net point two lists of points 
of $S$ that are mapped to it --- one list for positively labeled points and one for negatively labeled points. We now create an 
instance of vertex cover for the points of $S$: An edge $e=(u,v)$ for $u \in V^+$ and $v \in V^-$ 
is added to the edge set $E'$ if the distance between the respective net points $u'$ and $v'$ is at most 
$\left(1-\frac{2\eps}{\ddim(\X)} \right) \frac{2}{L}$. 
Notice that $E' \subset E$,
because the distance between such $u,v$ is at most
$\left(1-\frac{2\eps}{\ddim(\X)} \right) \frac{2}{L}
+ \frac{2\eps}{\ddim(\X)} \cdot \frac{2}{L} 
=  \frac{2}{L}$.
Moreover, the edge set $E'$ can be stored {\em implicitly} by recording every pair of net points that are within 
distance $\left(1-\frac{2\eps}{\ddim(\X)} \right) \frac{2}{L}$ --- 
oppositely labeled point pairs that map (respectively) to this net-point pair 
is considered (implicitly) to have an edge in $E'$. 
By the packing property, the number of net-point pairs to be recorded 
is at most $\left( {\ddim(\X)}/{\eps} \right)^{O(\ddim(\X))} n$, and by employing a hierarchy,
the entire (implicit) construction may be done in time 
$2^{O(\ddim(\X))} n \log n + \left( {\ddim(\X)}/{\eps} \right)^{O(\ddim(\X))} n$.

Now, for a given $L$, the run of the greedy algorithm for vertex cover
can be implemented on this graph in time 
$\left( {\ddim(\X)}/{\eps} \right)^{O(\ddim(\X))} n$, as follows.
The greedy algorithm considers a pair of net points within distance 
$\left(1-\frac{2\eps}{\ddim(\X)} \right) \frac{2}{L}$. 
If there exist $u \in V^+$ and $v \in V^-$ that map to these net points, 
then $u,v$ are deleted from $S$ and from the respective lists of the net points. 
(And similarly if $u,v$ map to the same net point.) 
The algorithm terminates when there are no more points to remove, and correctness follows.

We now turn to the analysis. Since $E' \subset E$, the
guarantees of the earlier greedy algorithm still hold. 
The resulting point set may contain opposite 
labeled points within distance 
$\left(1-\frac{2\eps}{\ddim(\X)} \right) \frac{2}{L}
- \frac{2\eps}{\ddim(\X)} \cdot \frac{2}{L}
= \left(1-\frac{4\eps}{\ddim(\X)} \right) \frac{2}{L}$, 
resulting in a Lipschitz constant $L/(1-\frac{4\eps}{\ddim(\X)})$. 
This Lipschitz constant is slightly larger than the given $L$,
which has the effect of increasing the complexity term in $G(L)$ by factor
$\left(1-\frac{4\eps}{\ddim(\X)} \right)^{-\ddim(\X)/2} = 1+\Theta(\eps)$.
The final result is achieved by scaling down $\eps$ to remove the leading constant.

\section{Example: Earthmover and time-series metrics}\label{sec:example}

To illustrate the potential power of our approach, 
we analyze its potential for two well-known metrics, 
the earthmover distance which operates on geometric sets, 
and Edit Distance with Real Penalty, which operates on time-series. 
We use the earthmover distance again in Section \ref{sec:exp}
for our experiments.

\myparagraph{Earthmover distance}
We will analyze the doubling dimension of an earthmover metric, which 
is a natural metric for comparing two \emph{sets} of $k$ geometric features.
It is often used in computer vision;
for instance, an image can be represented as a set of pixels in a color space,
yielding an accurate measure of dissimilarity between color characteristics 
of the images~\citep{RTG}.
In an analogous manner, an image can be represented as a set of
representative geometric features, 
such as object contours \citep{GD-contour}, other features~\citep{GD-kernel},
and SIFT descriptions \citep{PW08}.
In these contexts, $k\ge 2$ is usually a parameter 
which models the number of geometric features identified inside each image.

We use a simple yet common version, denoted $(\X_k,\EMD)$, 
where each point in $\X_k$ 
is a multiset of size $k$ in the unit square in the Euclidean plane, 
formally $S\subset [0,1]^2$ and $|S|=k$ (allowing and counting multiplicities).
The distance between such sets $S,T\in \X_k$ is given by 
\beqn
\label{eq:EMD}
  \EMD(S,T) =\min_{\pi:S\to T} \Big\{\tfrac1k \sum_{s\in S}
\|s-\pi(s)\|_2\Big\},
\eeqn
where the minimum is over all one-to-one mappings $\pi:S\to T$.
In other words, $\EMD(S,T)$ is the minimum-cost bipartite matching between 
the two sets $S,T$, where costs correspond to Euclidean distance.

\begin{lem}
The earthmover metric over $\X_k$ 
satisfies $\diam(\X_k) \le\sqrt 2$
and $\ddim(\X_k) \le O(k\log k)$.
\end{lem}
\begin{proof}
For the rest of this proof, a point refers to one in the unit square, not $\X_k$.
Consider a ball in $\X_k$ of radius $r>0$ around some $S$.
Let $N$ be an $r/2$-net of the unit square $[0,1]^2$,
according to the definition in Section \ref{sec:defn}.
Now consider all multisets $T\subset[0,1]^2$ of size $k$ 
that satisfy the following condition:
every point in $T$ belongs to the net $N$ and is within (Euclidean) distance
$(k+1/2)r$ from at least one point of $S$.
Points in such a multiset $T$ are chosen from a collection of size at most
$k\cdot \ceil{\tfrac{(k+1/2)r}{r/2}}^{O(1)}\le k^{O(1)}$ 
(by the packing property of the net points in the Euclidean plane).
Thus, the number of such multisets $T$ is 
$\lambda\le (k^{O(1)})^k = k^{O(k)}$.

To complete the proof of the lemma, it suffices to show that the radius $r$ ball 
(in $\X_k$) around $S$ is covered by the $\lambda$ balls of radius $r/2$ 
whose centers are given by the above multisets $T$.
To see this, consider a multiset $S'$ such that $\EMD(S,S')\le r$,
and let us show that $S'$ is contained in an $r/2$-ball around one of the above
multisets $T$.
Observe that every point in $S'$ is within distance at most $kr$ 
from at least one point of $S$. 
Now ``map'' each point in $S'$ to its nearest point in the net $N$,
which must be less than $r/2$ away, by the covering property of the net.
The result is a multiset $T$ as above with $\EMD(S',T)\le r/2$.
\end{proof}

\myparagraph{Time-series distance metric}
To present another example of the utility of our classification algorithms,
we show that a commonly used metric model of sparse time-series vectors
(with unbounded real coordinates) actually has a low doubling dimension.
A widely used similarity
function for time series is the Dynamic Time Warp (DTW) \cite{YJF-98},
a non-metric distance function between two time series
which is similar to the $\ell_1$ norm, except that it also allows coordinate
deletions or insertions in order to align the two series. The latter operations are used
to ensure that the resulting series are of equal length, and these operations can also
correlate the respective peaks and troughs of the series. We will consider a simple and
popular metric version of DTW known as Edit Distance with Real Penalty
(ERP) \cite{CN-04}, which allows for insertions of zero-valued elements only.

The ERP distance is formally defined as follows.
Given time-series vectors $r$ and $s$ with unbounded real coordinates and
where the length of the longer series is exactly $m$,
we may insert into $r$ and $s$
any number of zero-valued coordinates (called {\em gaps}) to produce new series
$\tilde{r}$ and $\tilde{s}$ of equal length.
Let $R_p$ be the set of all time series of length $p \ge m$
which may be derived from $r$ via gap insertions,
and similarly $S_p$ for $s$. Then
$d_{\ERP}(r,s) = \min _{p\ge m, \tilde{r} \in R_p,\tilde{s} \in S_p}
\| \tilde{r} - \tilde{s} \|_1$.
The ERP distance can be computed in quadratic time \cite{CN-04}.

Our contribution is twofold: We show in Lemma \ref{lem:ts-dense} that a set of time series of
length at most $m$ may have doubling dimension $\Omega(m)$ under ERP, even when the
coordinate range is limited to $\{1,2\}$. This dimension is quite
high, and motivates us to consider \emph{sparse} time series, which form an active field of study
\cite{FM-08,ZPS-10,GNS-12,KD-13,YGJ-13}. We show in Lemma \ref{lem:ts-sparse} that the set of
time series vectors with only $k$ non-zero elements --- that is, $k$-sparse
vectors --- has doubling dimension $O(k \log k)$ under ERP, irrespective of the vector 
length $m$ and even when the coordinates are real and unbounded. 
We first prove the claim below, and then proceed to the lower-bound 
on the dimension of length $m$ vectors under $\ERP$.

\begin{claim}\label{clm:vectors}
Consider the set $T = \{1,2\}^m$, and an integer $d\in[4,m/2]$.
Then every vector $r \in T$ is within ERP distance $d$ 
of fewer than $\left( \frac{3em}{d} \right)^{2d}$
other vectors of $T$.

\end{claim}

\begin{proof}
We may view ERP on the vectors of $T$ as a procedure transforming a vector $r \in T$ 
into some vector $s \in T$ as follows: The procedure inserts gaps in $r$ to produce vector
$\tilde{r}$, uses substitutions to transform $\tilde{r}$ to $\tilde{s}$, and then deletes
all gaps (i.e., zero-valued elements) from $\tilde{s}$ to produce vector $s$.
Here, the cost of a substitution from $\tilde{r}$ to $\tilde{s}$ is 
considered to be $|\tilde{r}_i - \tilde{s}_i|\in\{1,2\}$,
while the insertions and deletions entail no cost.
But without loss of generality, 
we may assume that $\tilde{r}_i$ and $\tilde{s}_i$ are not both gaps.
It follows that whenever $\tilde{r}_i$ is produced by a gap insertion
or $\tilde{s}_i$ is a gap coordinate to be deleted,
there must be a substitution from $\tilde{r}_i$ to $\tilde{s}_i$. 
Thus, if $d_\ERP(r,s) \le j$ for $r,s \in T$, 
then the ERP procedure includes at most $j$ substitutions, 
and consequently at most $j$ insertions and at most $j$ deletions.

For a fixed $r$, the vector $\tilde{r}$ can be produced in one of at most
$$
  \sum_{j=0}^{d} {m+j \choose j}
  < d\binom{3m/2}{d}
  \le d \left( \frac{3em}{2d} \right)^{d}
$$
possible ways of inserting $j\leq d$ gaps elements among the $m$ 
coordinates of $r$.
(Here we used the standard formula for combination with replacement.)
Having produced $\tilde{r}$, the vector $\tilde{s}$ can be produced in at most 
$$
  \sum_{j=0}^{d} \binom{m+d}{j} 2^d 
  < d \binom{3m/2}{d} 2^d
  \leq d \left( \frac{3em}{d} \right)^{d}
$$ 
possible ways via substitutions in $j\leq d$ elements, 
where a single substitution sets some $\tilde{s}_i$ to one of two possible 
values different from $\tilde{r}_i$.
Having produced $\tilde{s}$, 
the vector $s$ is produced by simply removing all gaps in $\tilde{s}$. 
It follows that there are fewer than
$ d^2 \left( \frac{3em}{2d} \right)^{d} \left( \frac{3em}{d} \right)^{d}
  \leq \left( \frac{3em}{d} \right)^{2d}
$
vectors of $T$ within distance $d$ of $r$, as claimed.
\end{proof}

\begin{lem}\label{lem:ts-dense}
There exists a set $S \subset \{1,2\}^m$ whose doubling dimension 
under the $\ERP$ metric is $\Omega(m)$.
\end{lem}

\begin{proof}
We will demonstrate that for all $m \ge 4 \cdot 35 = 140$,
there exists a set $S \subset \{1,2\}^m$ 
of cardinality $2^{m/2}$ with diameter at most $m$ 
and minimum interpoint distance at least $d=\lfloor \frac{m}{35} \rfloor$. 
As a consequence of Lemma \ref{lem:doublpack}, 
this $S$ has doubling dimension $\Omega(m)$.
Our proof uses a neighborhood counting argument similar to the 
one presented in \cite[Lemma 8]{BEKMRRS-03}.

We begin with the set $T = \{1,2\}^m$ of cardinality $2^m$. 
The maximum interpoint distance under ERP in $T$ is at most $m$, 
as this is the maximum distance under $\ell_1$ in $T$. 
By Claim \ref{clm:vectors} with $d=\lfloor \frac{m}{35} \rfloor$,
each vector $r \in T$ is within distance $d$ of fewer than
$\left( \frac{3em}{d} \right)^{2d} 
\le \left( 3e \frac{m}{m/35-1} \right)^{2m/35}
= \left( 3 \cdot 35 e \frac{1}{1-35/m} \right)^{2m/35}
\le (4 \cdot 35 e)^{2m/35}
< 2^{m/2}$ 
other points of $T$.
We now use $T$ to construct $S$ greedily, starting with the empty set and 
repeatedly placing in $S$ a point of $T$ at distance at 
least $d=\lfloor \frac{m}{35} \rfloor$ from all points currently in $S$. 
Each point added to $S$
invalidates fewer than $2^{m/2}$ other points in $T$ from appearing in 
$S$ in the future --- these are the points within distance $d$
of the new point. It follows that $S$ contains at least
$\frac{2^m}{2^{m/2}}=2^{m/2}$ points, as claimed.
\end{proof}

Although the doubling dimension of time-series vectors under ERP is
large, the situation for sparse vectors is much better, irrespective of the
vector length and even when the coordinates are unbounded reals.

\begin{lem}\label{lem:ts-sparse}
Every set $S$ of $k$-sparse time-series vectors of real unbounded coordinates 
and arbitrary length has doubling dimension $O(k \log k)$ under the $\ERP$
metric.
\end{lem}

\begin{proof}
Similar to \cite{GK-13}, let the \emph{density constant} $\mu(S)$ of $S$ be
the smallest number such that for every $p>0$, 
every ball of radius $p$ contains at most $\mu(S)$ points of $S$ 
at mutual distance strictly greater than $\frac{p}{2}$.
It is known that the doubling constant of $S$ is at most the density constant,
i.e., $\lambda(S) \leq \mu(S)$.
Indeed, for each radius $p$ ball in $S$, 
take a maximal set (with respect to containment)
of points at mutual distance strictly greater than $\frac{p}{2}$, 
and let each point be the center of a ball of radius $\frac{p}{2}$; 
the maximality implies that the small balls cover all points of the larger one. 
We get that every ball in $S$ can be covered by at most $\mu(S)$ 
balls of half the radius. 

It thus suffices to prove an upper bound on the density constant $\mu(S)$.
To this end, consider a subset $T \subset S$ such that for some $p>0$,
all interpoint distances in $T$ are in the range $(4p,8p]$.
In what follows, we prove an upper bound on $|T|$.
First, for each vector $r \in T$, discretize the vector by rounding down each
coordinate $r_i$ to the nearest multiple of $\frac{p}{k}$, producing new set $T'$. 
Done over all coordinates, the rounding alters interpoint ERP
distances by less than $2k \cdot \frac{p}{k} = 2p$ in total, 
and therefore all interpoint ERP distances in $T'$ are in the range $(2p,10p]$.
We further remove from each $r \in T'$
all zero-valued coordinates, and this has no effect on interpoint ERP distances.
The resulting set is $T'$ of the same size as $T$, i.e., $|T'|=|T|$.

To bound $|T'|$, fix an arbitrary vector $r \in T'$, and consider the
number of distinct discretized $k$-sparse vectors at ERP distance 
at most $10p$ from $r$. 
The ERP procedure may add up to $k$ gaps to $r$ to produce $\tilde{r}$, 
and there are 
$\sum_{j=0}^k {k+j \choose j} 
\le \sum_{j=0}^k {2k \choose j} 
< 2^{2k}
$
possible gap configurations. 
Note that the length of $\tilde{r}$ is at most $2k$.
Moving to the substitutions, since all vectors of $T'$ are discretized into 
multiples of $\frac{p}{k}$ and are at distance at most $10p$, 
we can view the substitutions as adding or removing from these coordinates 
weight in units (multiples) of $\frac{p}{k}$, 
and there are in total $\frac{10p}{p/k}=10k$ such units. 
If we view each substitution as accounting for a unit weight, 
and associate each coordinate with a sign that encodes whether 
the weight will be added or subtracted from that coordinate,
then there are at most $(2k)^{10k}\cdot 2^{2k}$ possible substitution 
configurations to produce $\tilde{s}$. 
Having produced $\tilde{s}$, the vector $s$ is produced from it
by removing all gap elements. Altogether, 
$2^{\ddim(S)} \leq \mu(S) = |T|=|T'| \le 2^{2k} (2k)^{10k} 2^{2k} = k^{O(k)}$, 
from which the lemma follows.
\end{proof}

\section{Experiments}\label{sec:exp}

\begin{figure*}[t]
\begin{center}
\scalebox{.5}{
\scalebox{.7}[0.1]{\includegraphics{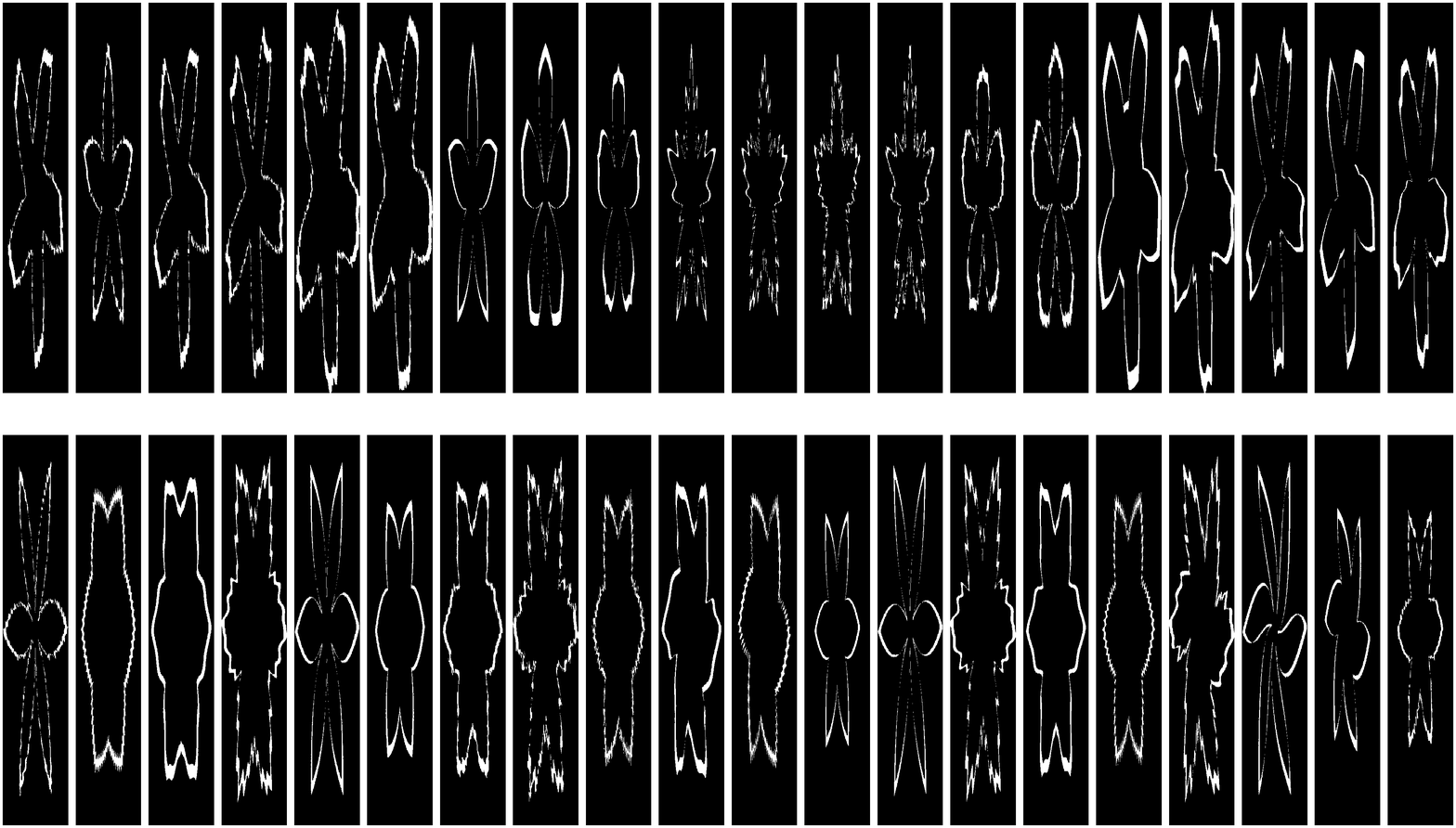}}}
\caption{%
The raw flower-contour data, $512\times512$ pixel black-and-white images.
The two kinds of flowers (five- and six-petaled) are displayed in separate rows.
}
\label{fig:flowers-clean}
\end{center}
\end{figure*}

We considered the task of distinguishing five-petaled flowers from six-petaled
ones. The images were taken from a shape matching/retrieval database 
called MPEG-7 Test Set\footnote{%
\url{http://www.dabi.temple.edu/~shape/MPEG7/dataset.html};
the 5-petaled flowers are under {\tt device0-1} 
and 6-petaled are under {\tt device1-1}.
},
and are displayed in Fig. \ref{fig:flowers-clean}.
The original images were represented as $512\times512$ black and white
matrices; we sampled these down to $128\times128$.
To render the task nontrivial, we retained only the image contour,
as otherwise, it would suffice to consider the ratio of black/white pixels
to achieve 100\% accuracy. 
To illustrate the relative advantage of earthmover distance
over the Euclidean one, we translated each image in the plane
by various random shifts. We ran four classification algorithms on this data:
\bit
\item {\bf Euclidean Nearest Neighbor.} 
The images were treated
as vectors in $\R^{128\times128}$, endowed with the Euclidean metric $\ell_2$.
\item {\bf EMD Nearest Neighbor.} 
The images were cut up into $16\times16=256$
square 
blocks, where each block $b$ is viewed as a vector $b\in\R^{8\times8}$.
Each image is thus represented as a sequence of $256$ blocks,
and over these sequences, EMD is defined as in (\ref{eq:EMD}) ---
except we used $\ell_1$ instead of $\ell_2$ as the base distance.
\item {\bf Euclidean SVM.} 
The Support Vector Machine (SVM) algorithm \cite{DBLP:journals/ml/CortesV95} 
was used, operating on vectors in $\R^{128\times128}$
with the Euclidean kernel $\iprod{x,y}=x\trn y$
and the regularization constant tuned by cross-validation.
\item {\bf SVM with RBF kernel.} 
The SVM algorithm was used, operating on vectors in $\R^{128\times128}$ 
with the Radial Basis Function (RBF) 
kernel $\iprod{x,y}=\exp(-\nrm{x-y}_2^2/\sigma^2)$,
where the regularization constant and $\sigma$ were 
tuned by cross-validation.
\eit

Our experimental results are listed in Table~\ref{tbl:experiments}. 
The relative magnitudes carry more significance than the absolute values,
as the latter fluctuate with experiment design choices,
such as the magnitude of the image translations, 
the thickness of the contour retained, and so forth.
These results exhibit a natural setting in which classification algorithms
for a non-Hilbertian metric significantly outperforms the Hilbert-space algorithms,
which is the main point we wished to illustrate here.

\begin{table}[t]
\begin{center}
\begin{tabular}{|l|c|}
\hline
Method                                    & Error \\ \hline
EMD nearest-neighbor 
                                          & 0.13        \\ 
Euclidean nearest-neighbor                & 0.39     \\ 
Euclidean SVM                             & 0.43        \\ 
SVM with RBF kernel                       & 0.39     \\ 
\hline
\end{tabular}
\caption{Experiments classifying the flower images.
Each method is averaged over hundreds of experiments.
}
\label{tbl:experiments}
\end{center}
\end{table}

\section*{Acknowledgements}
We thank the editor and anonymous referees for helpful
suggestions for the manuscript. 
Thanks also to Shai Ben-David
and Shai Shalev-Shwartz for sharing their book draft
and to Ulrike von Luxburg for enlightening discussions.

\bibliographystyle{plain}

\bibliography{ieee}

\end{document}